\documentclass[journal]{IEEEtai}

\usepackage[colorlinks,urlcolor=blue,linkcolor=blue,citecolor=blue]{hyperref}

\usepackage{amsmath}
\usepackage{amssymb}
\usepackage{amsthm}
\usepackage[date=year,doi=false,isbn=false,maxbibnames=99,style=ieee,url=false]{biblatex}
\usepackage{bm}
\usepackage{booktabs}
\usepackage{graphicx}
\usepackage{mathtools}
\usepackage{microtype}
\usepackage{multirow}
\usepackage{orcidlink}
\usepackage{subcaption}
\usepackage{thmtools}

\usetikzlibrary{positioning}

\newcommand{\header}[1]{\multicolumn{1}{c}{#1}}

\renewcommand\thesubsection{\arabic{subsection}}

\newcommand{\ack}{
    \section*{Acknowledgements}
    
    We thank Professor Yang Li from Tsinghua Shenzhen International Graduate School as well as anonymous reviewers for their feedback. This work was supported by Tsinghua University SIGS Start-up Fund under Grant QD2022024C, Shenzhen Science and Technology Innovation Commision under Grant JCYJ20220530143002005 and Shenzhen Ubiquitous Data Enabling Key Lab under Grant ZDSYS20220527171406015.
}
\newcommand{\authorinfo}{
    \author{Menghao Waiyan William Zhu\orcidlink{0009-0001-4180-5791}, Pengcheng Hao, and Ercan Engin Kuruoğlu
    \thanks{Menghao Waiyan William Zhu is with Tsinghua Shenzhen International Graduate School, Shenzhen, China (e-mail: zhumh22@mails.tsinghua.edu.cn).}
    \thanks{Pengcheng Hao is with Tsinghua Shenzhen International Graduate School, Shenzhen, China (e-mail: pengchenghao@sz.tsinghua.edu.cn).}
    \thanks{Ercan Engin Kuruoğlu is with Tsinghua Shenzhen International Graduate School, Shenzhen, China (e-mail: kuruoglu@sz.tsinghua.edu.cn).}
}}

\newcommand{\repourl}{\href{https://github.com/blackblitz/bcl}{\texttt{https://github.com/blackblitz/bcl}}}
\newcommand{\bio}{}

\newtheorem{proposition}{Proposition}
\newtheorem{lemma}{Lemma}
\graphicspath{{images/}}

\begin{document}

\title{
Sequential Function-Space Variational Inference via Gaussian Mixture Approximation
\thanks{This paper has been submitted to the IEEE Transactions on Artificial Intelligence for possible publication.}
}

\authorinfo

\markboth{}{}

\maketitle

\begin{abstract}
Continual learning in neural networks aims to learn new tasks without forgetting old tasks. Sequential function-space variational inference (SFSVI) uses a Gaussian variational distribution to approximate the distribution of the outputs of the neural network corresponding to a finite number of selected inducing points. Since the posterior distribution of a neural network is multi-modal, a Gaussian distribution could only match one mode of the posterior distribution, and a Gaussian mixture distribution could be used to better approximate the posterior distribution. We propose an SFSVI method based on a Gaussian mixture variational distribution. We also compare different types of variational inference methods with a fixed pre-trained feature extractor (where continual learning is performed on the final layer) and without a fixed pre-trained feature extractor (where continual learning is performed on all layers). We find that in terms of final average accuracy, likelihood-focused Gaussian mixture SFSVI outperforms other sequential variational inference methods, especially in the latter case.
\end{abstract}

\begin{IEEEImpStatement}
This paper presents work whose goal is to advance the field of continual learning. The ability to continually learn new knowledge without forgetting old knowledge has a great impact on society, especially in healthcare, where, for example, we would like to be able to learn new diseases without forgetting old diseases, or diagnosis and treatment is made as signs and symptoms progress over the course of time. Since continual learning methods inherently have some form of memory, it is important that they are designed to preserve the privacy of people.
\end{IEEEImpStatement}


\section{Introduction}

\IEEEPARstart{W}{hen} neural networks learn new tasks, they lose predictive performance on old tasks. This is known as catastrophic forgetting \parencite{mccloskey_catastrophic_1989}. It can be prevented by training using all the data, which we refer to as joint maximum a posteriori (MAP) training. However, access to the previous data may be limited due to computational or privacy constraints, so the aim of continual learning is to prevent forgetting with little or no access to the previous data.

For classification tasks, three types of continual learning settings are commonly studied \parencite{van_de_ven_three_2022}:
\begin{enumerate}
    \item\textbf{Task-incremental learning}, in which task IDs are provided and the classes change between tasks
    \item\textbf{Domain-incremental learning}, in which task IDs are not provided and the classes remain the same between tasks but the input data distribution changes between tasks
    \item\textbf{Class-incremental learning}, in which task IDs are not provided and the classes change between tasks
\end{enumerate}
In task-incremental learning, a multi-headed neural network is typically used, in which there is one head per task and the task ID is used to determine the head to use for prediction. In domain- and class-incremental learning, a single-headed neural network is typically used.

Task IDs make continual learning much easier by restricting the possible classes and are unlikely to be accessible in practice. We adhere to the desiderata proposed in \cite{farquhar_towards_2018} and restrict ourselves to domain- and class-incremental learning with single-headed neural networks on task sequences with more than two similar tasks and limited access to the previous data.

Continual learning methods are mostly based on regularization, replay or both. Regularization constrains the parameters of the neural network during training. Parameter-space regularization constrains them by directly penalizing their change \parencite{kirkpatrick_overcoming_2017,zenke_continual_2017,ritter_online_2018,nguyen_variational_2018}, while function-space regularization constrains them by penalizing the change in the outputs \parencite{li_learning_2016,pan_continual_2020,titsias_functional_2020,rudner_continual_2022}. Replay trains the neural network with the current data together with the previous data which may be stored or generated \parencite{chaudhry_continual_2019,rebuffi_icarl_2017}. It has been empirically observed that replay-based methods generally perform much better than regularization-based methods \parencite{mai_online_2022,farquhar_towards_2018,farquhar_unifying_2018}.

A class of continual learning method is based on sequential Bayesian inference on the parameters of the neural network, in which the previous posterior distribution becomes the current prior distribution, and prediction is provided by the posterior predictive distribution. Sequential maximum a posteriori inference (SMI) performs the prediction by using only the maximum of the posterior probability density function (PDF), while sequential variational inference (SVI) finds an approximate posterior PDF called the variational PDF by minimizing its Kullback-Leibler (KL) divergence from the posterior PDF and uses it to give an approximate posterior predictive distribution.

\cite{farquhar_unifying_2018} describes two general ways of performing SVI. In a prior-focused approach, the previous variational distribution is used as the current prior distribution, while in a likelihood-focused approach, the prior distribution is not updated, and the previous likelihood functions are approximated by using the previous data which may be stored or generated. Thus, prior-focused approaches are based on regularization, and likelihood-focused approaches are based on replay.

Neural network posteriors are known to be highly complex and often multi-modal \parencite{izmailov_what_2021}. To capture this structure, we employ a Gaussian mixture variational distribution, which provides greater flexibility than standard Gaussian approximations. This choice is further motivated by prior work showing that various heavy-tailed distributions, including \(\alpha\)-stable distributions, can be expressed as scale mixtures of Gaussians \parencite{kuruoglu_new_1997}, and have been used to model the distributional behavior of neural network parameters \parencite{li_exploring_2024}.

We propose a method for learning a Gaussian mixture variational distribution via a function-space approach. Our main goal is to investigate whether Gaussian mixture methods perform better than Gaussian methods. We also combine likelihood-focused and function-space approaches and compare prior-focused and likelihood-focused approaches. Finally, since a fixed pre-trained feature extractor can be used to reduce the high computational cost of function-space approaches, we also investigate the usefulness of a fixed pre-trained feature extractor for these methods.

The rest of the paper of organized as follows. Section \ref{sec:background} introduces the background for sequential Bayesian inference and sequential variational inference, and Section \ref{sec:relwork} mentions related work. Section \ref{sec:gmsfsvi} describes the proposed method, and Section \ref{sec:exp} describes the experiments. Finally, Section \ref{sec:conc} provides the conclusion. Appendix \ref{sec:expdetails} provides the experiment details.

\section{Background}
\label{sec:background}

We introduce sequential Bayesian inference and describe three SVI methods, which our proposed method builds upon. We use bold letters to denote random variables, and the difference between vectors and scalars should be clear from context.

\subsection{Sequential Bayesian Inference}

We describe a probabilistic model for Bayesian continual learning. Let \(\bm\theta\) be the collection of parameters of the neural network and \((\bm x_t)_{t=1}^T\) and \((\bm y_t)_{t=1}^T\) be the inputs and outputs from time \(1\) to \(t\), respectively. \((\bm x_t)_{t=1}^T\) are assumed to be independent, and \((\bm y_t)_{t=1}^T\) are assumed to be conditionally independent given \(\bm\theta\) and \((\bm x_t)_{t=1}^T\). These assumptions are depicted in \ref{fig:probmodel}.

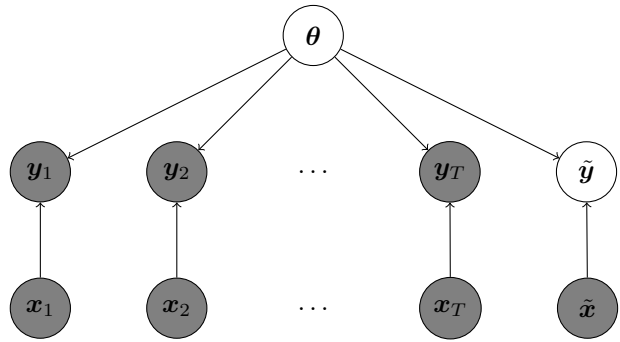
\begin{figure}[ht]
    \centering
    \centering
        \begin{tikzpicture}[inferrable/.style={circle, draw}, observable/.style={circle, draw, fill=gray}, minimum size=8mm, node distance=10mm]
            \node[inferrable] (theta) {\(\bm\theta\)};
            \node (ydots) [below=of theta] {\(\ldots\)};
            \node (xdots) [below=of ydots] {\(\ldots\)};
            \node[observable] (y2) [left=of ydots] {\(\bm y_2\)};
            \node[observable] (x2) [left=of xdots] {\(\bm x_2\)};
            \node[observable] (y1) [left=of y2] {\(\bm y_1\)};
            \node[observable] (x1) [left=of x2] {\(\bm x_1\)};
            \node[observable] (yt) [right=of ydots] {\(\bm y_T\)};
            \node[observable] (xt) [right=of xdots] {\(\bm x_T\)};
            \node[inferrable] (ypred) [right=of yt] {\(\tilde{\bm y}\)};
            \node[observable] (xpred) [right=of xt] {\(\tilde{\bm x}\)};
            \draw[->] (theta) -- (y1);
            \draw[->] (theta) -- (y2);
            \draw[->] (theta) -- (yt);
            \draw[->] (theta) -- (ypred);
            \draw[->] (x1) -- (y1);
            \draw[->] (x2) -- (y2);
            \draw[->] (xt) -- (yt);
            \draw[->] (xpred) -- (ypred);
        \end{tikzpicture}
    \caption{Bayesian network for continual learning. \(\bm\theta\) is the collection of parameters of the neural network. \((\bm x_t)_{t=1}^T\) are the inputs, and \((\bm y_t)_{t=1}^T\) are the outputs. \(\tilde{\bm x}\) and \(\tilde{\bm y}\) are the input and output for prediction, respectively. Shaded nodes are observed.}
    \label{fig:probmodel}
\end{figure}

Note that \((\bm x_t,\bm y_t)_{t=1}^T\) are not necessarily identically distributed, but tasks are similar, i.e. the form of the likelihood function \(l_t(y_t|\theta,x_t)\) is the same for all tasks. For example, in multi-class classification, it is the categorical likelihood function for all tasks. \(\bm\theta\) is assumed to be a continuous random variable of an appropriate dimension, so it has a PDF.

For \(t=1,2,\ldots,T\), let \(\bm D_t=(\bm x_\tau,\bm y_\tau)_{\tau=1}^t\) be all the training data observed until time \(t\). The posterior PDF of \(\bm\theta\) given the training data \(\bm D_t=D_t\) is provided by Bayes' rule:
\begin{equation}
    p_t(\theta|D_t)=\frac1{z_t}p_{t-1}(\theta|D_{t-1})l_t(y_t|\theta,x_t)\label{eq:bayes}
\end{equation}
where \(z_t=\int_\Theta p_{t-1}(\theta|D_{t-1})l_t(y_t|\theta,x_t)d\theta\) is a normalization constant.

The posterior predictive function of a testing output \(\tilde{\bm y}\) given a testing input \(\tilde{\bm x}=\tilde x\) and the training data \(\bm D_t=D_t\) is the expectation of the likelihood function with respect to the posterior distribution of \(\bm\theta\):
\begin{equation}
    \tilde p_t(\tilde y|\tilde x,D_t)=\int_\Theta p_t(\theta|D_t)l_t(\tilde y|\theta,\tilde x)d\theta=\mathbf E_{p_t(\cdot|D_t)}[l_t(\tilde y|\bm\theta,\tilde x)]
\end{equation}
In classification, the likelihood function is a categorical probability mass function (PMF) represented by the predictive class probabilities, which can be averaged across a posterior sample from \(p_t(\cdot|D_t)\) to give an approximate posterior predictive PMF. This is also known as Bayesian model averaging.

\subsection{Gaussian Variational Continual Learning}

Gaussian variational continual learning (G-VCL) \cite{nguyen_variational_2018} approximates the posterior PDF of \(\bm\theta\) at time \(t\) with a Gaussian variational distribution \(q_t\) by minimizing its KL divergence from the posterior PDF of \(\bm\theta\), which is equivalent to minimizing a loss function called the variational free energy (VFE) with respect to the parameter collection \(\phi\) of the variational PDF:
\begin{equation}
\mathfrak L_t(\phi)=\mathbf E_{q_t}[\mathfrak l_t(\bm\theta)]+D_{KL}\left(q_t\Vert q_{t-1}\right)\label{eq:pvfe}
\end{equation}
where \(\mathfrak l_t\) is the negative log likelihood function at time \(t\) and \(q_0\) is the initial prior PDF. The training method is known as Bayes by Backprop \cite{blundell_weight_2015}. The covariance matrix is assumed to be diagonal, with the \(j\)-th diagonal entry given by  \(\sigma_j^2=(\mathrm{softplus}(\rho_j))^2\) for some real number \(\rho_j\), which we refer to as the \(j\)-th modified standard deviation. Then, the parameter collection \(\phi=(\mu,\rho)\) includes the means and the modified standard deviations. Here, the standard deviation is modified so that we can perform unconstrained optimization.

To perform gradient descent on Equation \ref{eq:pvfe}, a "reparameterization trick" is used to compute the expectation, which is essentially rewriting the neural network parameter \(\bm\theta\) as a function of the variational parameters \(\phi\) and a standard random variable:
\begin{equation}
    \bm\theta=\mu+\mathrm{softplus}(\rho)\odot\bm z\label{eq:rt_g}
\end{equation}
where \(\odot\) is element-wise multiplication and \(\bm z\) is the standard Gaussian random variable with the same dimension as \(\bm\theta\). The expectation is approximated by taking a sample of \(\bm z\), applying Equation \ref{eq:rt_g} to obtain a sample of \(\bm\theta\) and computing the average of \(\mathfrak l_t(\bm\theta)\) across the sample. The KL divergence can be computed in closed form or approximated similarly. In mini-batch gradient descent, it is divided by the number of batches in the dataset, which is known as "KL reweighting".

There are two variants \parencite{farquhar_unifying_2018}:
\begin{enumerate}
    \item\textbf{Prior-focused}: The variational PDF of the current task is used as the prior PDF of the next task. The VFE is as in Equation \ref{eq:pvfe}.
    \item\textbf{Likelihood-focused}: The prior PDF is not updated, and a coreset, which is a small collection of stored previous data, is used to approximate the previous likelihood functions. The VFE becomes
        \begin{equation}
            \mathfrak L_t(\phi)=\mathbf E_{q_t}\left[\sum_{i=1}^t\mathfrak l_i(\bm\theta)\right]+D_{KL}(q_t\Vert q_0)\label{eq:lvfe}
        \end{equation}
\end{enumerate}

\subsection{Gaussian Mixture Variational Continual Learning}

Gaussian mixture variational continual learning (GM-VCL) \cite{phan_reducing_2022} uses a Gaussian mixture variational distribution to perform parameter-space variational inference. A Gaussian mixture distribution is a categorical mixture of Gaussian distributions, i.e., there is a categorical random variable \(\bm\kappa\) with \(k\) categories such that given \(\bm\kappa=\kappa\), \(\bm\theta\) is conditionally Gaussian. It has a categorical PMF with probabilities \((p_i)_{i=1}^k\), and we refer to \(\bm\kappa\) as the mixing categorical random variable. Thus, the variational distribution has PDF \(\sum_{\kappa=1}^kp_\kappa f_\kappa(\theta)\), where \(f_\kappa\) is the Gaussian PDF of the \(\kappa\)-th component.

In order to make Equation \ref{eq:pvfe} differentiable, the categorical distribution is approximated with a Gumbel-softmax distribution \cite{jang_categorical_2017,maddison_concrete_2017}, which is a continuous probability distribution over the probability simplex in \(\mathbb R^k\). The reparameterization trick for the Gumbel-softmax random variable \(\tilde{\bm\kappa}\) is
\begin{equation}
    \tilde{\bm\kappa}=\mathrm{softmax}\left(\frac1T(\lambda+\bm g)\right)\label{eq:rt_gm_k}
\end{equation}
where \(T\) is a parameter known as the temperature, \(\lambda\) is a real vector of un-normalized log probabilities of the categories and \(\bm g\) is a standard Gumbel random variable in \(\mathbb R^k\). \(\tilde{\bm\kappa}\) converges in distribution to \(\bm\kappa\) as \(T\to0\), so a small \(T\) can be used for the approximation. Then, the reparameterization trick for \(\bm\theta\) is
\begin{equation}
    \bm\theta=\sum_{i=1}^k\tilde{\bm\kappa}_i\bm\theta_i\label{eq:rt_gm}
\end{equation}
where \(\tilde{\bm\kappa}_i\) is the \(i\)-th component of \(\tilde{\bm\kappa}\) and \(\bm\theta_i\) is the \(i\)-th Gaussian component, i.e. \(\bm\theta\) given \(\bm\kappa=\kappa_i\).

The expectation in Equation \ref{eq:pvfe} can be approximated by using the reparameterization trick in \ref{eq:rt_gm}. The KL divergence for Gaussian mixture distributions does not have a closed form expression, so it can be approximated in the same way. Alternatively, it may be replaced with an upper bound which has a closed-form expression:
\begin{equation}
    \begin{split}
        \tilde D_{KL}(q_t\Vert q_{t-1})=&D_{KL}(q_{t,\bm\kappa}\Vert q_{t-1,\bm\kappa})\\&+\sum_{\kappa=1}^kp_{t,\kappa}D_{KL}(q_{t,\bm\theta_\kappa}\Vert q_{t-1,\bm\theta_\kappa})
    \end{split}\label{eq:klub}
\end{equation}
where \(q_{t,\bm\kappa}\) is the variational mixing categorical PMF at time \(t\), \(p_{t,\kappa}\) is the \(\kappa\)-th variational mixing categorical probability at time \(t\) and \(q_{t,\bm\theta_\kappa}\) is the variational PDF of the \(\kappa\)-th Gaussian component at time \(t\). In words, the first term is the KL divergence of the current mixing categorical PMF from the previous mixing categorical PMF and the second term is a weighted sum of the KL divergences of the current Gaussian PDFs from the previous Gaussian PDFs, weighted by the current mixing categorical probabilities. This upper bound was originally stated in \cite{singer_batch_1998}.

\subsection{Gaussian Sequential Function-Space Variational Inference}

Gaussian sequential function-space variational inference (G-SFSVI) \cite{rudner_continual_2022} aims to directly approximate the posterior predictive distribution of \(\bm f(x)=f(x;\bm\theta)\), where \(f\) is the neural network function before the final activation function. \(\bm f(x)\) can be viewed as a random process indexed by the input \(x\), i.e. for fixed \(x\), it is a random variable. We perform an affine approximation of \(f\) around \(\mu\):
\begin{equation}
\tilde f(x;\bm\theta)=f(x;\mu)+J_\mu(x)(\bm\theta-\mu)\label{eq:affine}
\end{equation}
where \(J_\mu(x)\) is the Jacobian matrix of \(f(x;\theta)\) with respect to \(\theta\) at \(\mu\). Since \(\bm\theta\) has a Gaussian variational PDF, \(\tilde{\bm f}(x)=\tilde f(x;\bm\theta)\) is a Gaussian process indexed by \(x\) with mean function \(\tilde\mu(x)=f(x;\mu)\) and covariance function \(\tilde\Sigma(x_1,x_2)=J_\mu(x_1)\Sigma(J_\mu(x_2))^T\), where \(\Sigma\) is the diagonal covariance matrix.

Then, the VFE involves a KL divergence of Gaussian processes, which is intractable. To make it tractable, a technical assumption called "prior conditional matching" is made \cite{rudner_rethinking_2021}, and it is replaced with a KL divergence of Gaussian PDFs of a finite number of inducing points:
\begin{equation}
\mathbf E_{q_t}[\mathfrak l_t(\bm\theta)]+D_{KL}(\tilde q_t\Vert\tilde q_{t-1})\label{eq:psfsvivfe}
\end{equation}
where \(\tilde q_t\) is the variational PDF at time \(t\) of the inducing outputs \(\tilde{\bm f_I}=(\tilde{\bm f}(x_j))_{j=1}^n\), where \((x_j)_{j=1}^n\) are the inducing inputs. The KL divergence can be computed in closed form, but in order to simplify the computation, diagonalization is made so that the covariances between different inducing points and between different output components are ignored. In particular, for all pairs of inducing points \((x_1,x_2)\), in the output covariance matrix \(\tilde\Sigma(x_1,x_2)=J_\mu(x_1)\Sigma(J_\mu(x_2))^T\), all the entries are set to zero if \(x_1\ne x_2\), and all the off-diagonal entries are set to zero if \(x_1=x_2\).

\section{Related Work}
\label{sec:relwork}

In Section \ref{sec:background}, we have already described three related works based on SVI, namely Gaussian variational continual learning (G-VCL) \parencite{nguyen_variational_2018}, Gaussian mixture variational continual learning (GM-VCL) \parencite{phan_reducing_2022} and Gaussian sequential function-space variational inference (G-SFSVI) \parencite{rudner_continual_2022}. There are also some other works that use the function-space approach. Functional regularization of the memorable past \parencite{pan_continual_2020} uses a Laplace approximation to approximate the output distribution. Functional regularization for continual learning \parencite{titsias_functional_2020} uses variational inference on the final layer and treats the earlier layers as deterministic parameters to be optimized.

Another approach to Bayesian continual learning is based on sequential MAP inference (SMI), typically using a quadratic approximation of the previous loss function as a regularization term, so these methods rely on approximating the Hessian matrix. Elastic weight consolidation (EWC) uses a diagonal approximation of the empirical Fisher information matrix (eFIM) \parencite{kirkpatrick_overcoming_2017}, adding a quadratic term to the objective for every task. \parencite{huszar_quadratic_2017,huszar_note_2018} propose a corrected objective with a single quadratic term for which the eFIM can be cumulatively added. Synaptic intelligence (SI) uses the change in loss during optimization to give a diagonal approximation \cite{zenke_continual_2017}. Online structured Laplace approximation uses Kronecker factorization to perform a block-diagonal approximation \cite{ritter_online_2018}. Experience replay (ER) is a simple method which uses a coreset for MAP training \parencite{chaudhry_continual_2019}.

Our proposed method performs function-space variational inference using a Gaussian mixture variational distribution. Unlike GM-VCL \parencite{phan_reducing_2022}, it performs function-space variational inference, and unlike G-SFSVI \parencite{rudner_continual_2022}, it uses a Gaussian mixture variational distribution. Moreover, we extend the idea of the likelihood-focused variational inference \parencite{farquhar_unifying_2018} to function-space variational inference.

\section{Gaussian Mixture Sequential Function-Space Variational Inference}
\label{sec:gmsfsvi}

We propose a method which performs sequential function-space variational inference with a Gaussian mixture variational PDF of \(\bm\theta\) with \(k\) components. Let \(\bm\kappa\) be the mixing categorical random variable with mixing categorical probabilities \((p_i)_{i=1}^k\). We perform a conditional affine approximation of \(f\) around \(\mu_\kappa\) given \(\bm\kappa=\kappa\), where \(\kappa=1,2,\ldots,k\):

\begin{equation}
\tilde f(x;\bm\theta)=f(x;\mu_\kappa)+J_{\mu_\kappa}(x)(\bm\theta-\mu_\kappa)\label{eq:condaffine}
\end{equation}

We will show that \(\tilde{\bm f}(x)=\tilde f(x;\bm\theta)\) in Equation \ref{eq:condaffine} is a mixture of Gaussian processes. First, we introduce a lemma.

\begin{lemma}
Let \(\bm\theta\) be an \(\mathbb R^n\)-valued Gaussian mixture random variable with \(k\) components, the \(\kappa\)-th component having mixing probability \(p_\kappa\), mean vector \(\mu_\kappa\) and covariance matrix \(\Sigma_\kappa\). Let \(\bm\kappa\) be the mixing categorical random variable. Consider a conditional affine transformation \(T_{\bm\kappa}:\mathbb R^n\to\mathbb R^d\) such that given \(\bm\kappa=\kappa\), \(T_\kappa(\bm\theta)=A_\kappa\bm\theta+b_\kappa\), where \(A_\kappa\) is a \(d\times n\) real matrix and \(b_\kappa\) is a vector in \(\mathbb R^d\). Then, \(T_{\bm\kappa}(\bm\theta)\) is an \(\mathbb R^d\)-valued Gaussian mixture random variable with \(k\) Gaussian components, the \(\kappa\)-th component having mixing probability \(p_\kappa\), mean vector \(A_\kappa\mu_\kappa+b_\kappa\),  covariance matrix \(A_\kappa\Sigma_\kappa A_\kappa^T\).\label{lem:condaffine}
\end{lemma}

\begin{proof}
The characteristic function of \(\bm\theta\) is
\[
\begin{aligned}
\phi_{\bm\theta}(t)&=\bm E\left[e^{it^T{\bm\theta}}\right]\\
&=\mathbf E\left[\bm E\left[\left.e^{it^T{\bm\theta}}\right|\bm\kappa\right]\right]\\
&=\sum_{\kappa=1}^kp_{\kappa}e^{{it^T\mu_\kappa}-\frac12t^T\Sigma_\kappa t}\\
\end{aligned}
\]
The second equality follows from the law of iterated expectations. The third equality follows from the definition of expectation and the characteristic function of a Gaussian random variable.

The characteristic function of \(T_{\bm\kappa}(\bm\theta)\) is
\[
\begin{aligned}
\phi_{T_{\bm\kappa}(\bm\theta)}(t)&=\bm E\left[e^{it^T(T_{\bm\kappa}(\bm\theta))}\right]\\
&=\mathbf E\left[\mathbf E\left[\left.e^{it^T(T_{\bm\kappa}(\bm\theta))}\right|\bm\kappa\right]\right]\\
&=\mathbf E\left[\mathbf E\left[\left.e^{it^T(A_{\bm\kappa}\bm\theta+b_{\bm\kappa})}\right|\bm\kappa\right]\right]\\
&=\sum_{\kappa=1}^kp_\kappa e^{{it^T(A_\kappa\mu_\kappa+b_\kappa)}-\frac12t^TA_\kappa\Sigma_\kappa A_\kappa^Tt}\\
\end{aligned}
\]
The second equality follows from the law of iterated expectations. The third equality follows from the definition of \(T_{\bm\kappa}\). The fourth equality follows from the definition of the expectation and the characteristic function of an affine transformation of a Gaussian random variable.

Comparing with \(\phi_{\bm\theta}\), it is clear that \(\phi_{T_{\bm\kappa}(\bm\theta)}\) is the characteristic function of a Gaussian mixture random variable with the required parameters. Therefore, \(T_{\bm\kappa}(\bm\theta)\) is a Gaussian mixture random variable with the required parameters.
\end{proof}

Now, the following proposition states that \(\tilde{\bm f}(x)=\tilde f(x;\bm\theta)\) in Equation \ref{eq:condaffine} is a mixture of Gaussian processes.

\begin{proposition}
\(\tilde{\bm f}(x)=\tilde f(x;\bm\theta)\) defined in Equation \ref{eq:condaffine} is a mixture of Gaussian processes indexed by \(x\) with \(k\) components, the \(\kappa\)-th component having mixing probability \(p_\kappa\), mean function \(\tilde\mu_\kappa(x)=f(x;\mu_\kappa)\) and covariance function \(\tilde\Sigma_\kappa(x_1,x_2)=J_{\mu_\kappa}(x_1)\Sigma_\kappa(J_{\mu_\kappa}(x_2))^T\).\label{prop:mgp}
\end{proposition}

\begin{proof}
Given \(\bm\kappa=\kappa\), \(\tilde f(x;\bm\theta)\) can be rewritten as \(J_{\mu_\kappa}(x)\bm\theta+(f(x;\mu_\kappa)-J_{\mu_\kappa}(x)\mu_\kappa)\), so it is clear that \(\tilde f(x;\bm\theta)\) is a conditional affine transformation of \(\bm\theta\). By Lemma \ref{lem:condaffine}, it is a mixture of Gaussian processes with the required parameters.
\end{proof}

Similarly to G-SFSVI, the KL divergence of mixtures of Gaussian processes is replaced with a KL divergence of Gaussian mixture PDFs of a finite number of inducing points as in Equation \ref{eq:psfsvivfe}. The KL divergence cannot be computed in closed form. It can be either computed using Monte Carlo integration or replaced with an upper bound as in GM-VCL, leading to the following VFE:
\begin{equation}
    \begin{split}
        \mathfrak L_t(\phi)=&\mathbf E_t[\mathfrak l_t(\bm\theta)]+D_{KL}(q_{t,\bm\kappa}\Vert q_{t-1,\bm\kappa})\\&+\sum_{\kappa=1}^kp_{t,\kappa}D_{KL}(\tilde q_{t,\bm\theta_\kappa}\Vert\tilde q_{t-1,\bm\theta_\kappa})
    \end{split}\label{eq:pgmsfsvivfe}
\end{equation}
where \(\mathfrak l_t\) is the negative log likelihood with respect to the data at time \(t\), \(q_{t,\bm\kappa}\) is the variational mixing categorical PMF at time \(t\), \(p_{t,\kappa}\) is the \(\kappa\)-th variational mixing categorical probability at time \(t\) and \(\tilde q_{t,\bm\theta_\kappa}\) is the variational PDF of the \(\kappa\)-th Gaussian component at time \(t\) of the inducing outputs \(\tilde{\bm f_I}=(\tilde{\bm f}(x_j))_{j=1}^n\), where \((x_j)_{j=1}^n\) are the inducing inputs. Here, \(\phi=(\lambda,(\mu_\kappa)_{\kappa=1}^k,(\rho_\kappa)_{\kappa=1}^k)\), where \(\lambda\) is the collection of log un-normalized probabilities of the mixing categorical PMF, \((\mu_\kappa)_{\kappa=1}^k\) and \((\rho_\kappa)_{\kappa=1}^k\) are the means and modified standard deviations of the \(k\) Gaussian components.

By Proposition \ref{prop:mgp}, the \(\kappa\)-th component has mixing probability \(p_\kappa\), mean \(\tilde\mu_\kappa(x)=f(x;\mu_\kappa)\) for every inducing point \(x\) and covariance matrix \(\tilde\Sigma_\kappa(x_1,x_2)=J_{\mu_\kappa}(x_1)\Sigma_\kappa(J_{\mu_\kappa}(x_2))^T\) for every pair of inducing points \((x_1,x_2)\). As in G-SFSVI, diagonalization is made so that the covariances between different inducing points and between different output components are ignored. After the diagonalization, the \(\kappa\)-th Gaussian component of the outputs has a mean and a variance, each of dimension \(nd\), where \(n\) is the number of inducing points and \(d\) is the output dimension, i.e. the number of classes in classification tasks. The KL divergence of the \(\kappa\)-th Gaussian PDFs can be computed in closed form:
\begin{equation}
    \begin{split}
    &D_{KL}(\tilde q_{t,\bm\theta_\kappa}\Vert\tilde q_{t-1,\bm\theta_\kappa})\\=&\frac12\sum_{i=1}^{nd}\left(\ln\frac{\sigma_{t-1,\kappa,i}^2}{\sigma_{t,\kappa,i}^2}-1\right)\\&+\frac12\sum_{i=1}^{nd}\frac{\sigma_{t,\kappa,i}^2+(\mu_{t,\kappa,i}-\mu_{t-1,\kappa,i})^2}{\sigma_{t-1,\kappa,i}^2}
    \end{split}\label{eq:klgauss}
\end{equation}

The KL divergence of the mixing categorical PMFs can also be computed in closed form:
\begin{equation}
    D_{KL}(q_{t,\bm\kappa}\Vert q_{t-1,\bm\kappa})=\sum_{\kappa=1}^kp_{t,\kappa}\ln\frac{p_{t,\kappa}}{p_{t-1,\kappa}}\label{eq:klcat}
\end{equation}

The reparameterization trick for \(\bm\theta\) is as in Equation \ref{eq:rt_gm}, where the mixing categorical distribution is approximated with a Gumbel-softmax distribution in order to make Equation \ref{eq:pgmsfsvivfe} differentiable. When using mini-batch gradient descent to optimize the objective, the KL divergence is scaled by multiplying with the ratio of the current data batch size to the coreset batch size. Note that this is different from the "KL reweighting" in parameter-space approaches, where we divide the KL divergence by the number of batches.

The prior-focused and likelihood-focused variants described in \cite{farquhar_unifying_2018} can be extended to function-space methods as follows:
\begin{enumerate}
    \item\textbf{Prior-focused}: The variational posterior PDF of the current task is used as the prior PDF of the next task. The VFE is as in Equation \ref{eq:psfsvivfe}. For the first task, the inducing points are randomly generated, but starting from the second task, the inducing points are randomly sampled from the coreset. This is the approach used in \cite{rudner_continual_2022}.
    \item\textbf{Likelihood-focused}: The prior PDF is not updated, and a coreset, which is a small collection of stored previous data, is used to approximate the previous likelihood functions. The inducing points are randomly generated for all tasks. The VFE becomes
        \begin{equation}
            \mathfrak L_t(\phi)=\mathbf E_{q_t}\left[\sum_{i=1}^t\mathfrak l_i(\bm\theta)\right]+D_{KL}(\tilde q_t\Vert\tilde q_0)\label{eq:lsfsvivfe}
        \end{equation}
\end{enumerate}

Assuming that the initial prior PDF is standard Gaussian, the training for the two variants of our method consists of three stages per task:
\begin{enumerate}
    \item\textbf{Prior-focused GM-SFSVI}:
    \begin{enumerate}
        \item\textbf{Define the VFE}: The VFE in Equation \ref{eq:pgmsfsvivfe} is used. The initial collection of log un-normalized mixing probabilities \(\lambda\) is \(0\), the initial prior means \((\mu_\kappa)_{\kappa=1}^k\) are all \(0\) and the initial prior modified standard deviations \((\rho_\kappa)_{\kappa=1}^k\) are all such that \(\sigma_\kappa\) consists entirely of ones. Starting from the second task, the VFE is updated with the prior parameters.
        \item\textbf{Minimize the VFE}: Training is done on the VFE using mini-batch gradient descent. The mini-batch is used for computing the expectation, and inducing points are used for computing the KL divergence. For the first task, the inducing points are randomly generated, e.g. from a uniform probability distribution whose boundaries are known in advance or determined from the data. Starting from the second task, they are randomly sampled from the coreset. 
        \item\textbf{Update the coreset and prior parameters}: The coreset is updated by including a fixed number of data points that are randomly selected from the current task, and the prior parameters are updated with the current variational parameters.
    \end{enumerate}
    \item\textbf{Likelihood-focused GM-SFSVI}
    \begin{enumerate}
        \item\textbf{Define the VFE}: The VFE in Equation \ref{eq:lsfsvivfe} (with the KL divergence replaced with the upper bound in Equation \ref{eq:klub}) is used. The initial prior parameters are as in prior-focused GM-SFSVI. No update is necessary unless the KL weight changes between tasks.
        \item\textbf{Minimize the VFE}: Training is done on the VFE using mini-batch gradient descent. The mini-batch is concatenated with the coreset mini-batch, which are together used for computing the expectation, and inducing points, which are randomly generated as for the first task of prior-focused GM-SFSVI, are used for computing the KL divergence.
        \item\textbf{Update the coreset}: The coreset is updated by including a fixed number of data points that are randomly selected from the current task.
    \end{enumerate}
\end{enumerate}

\section{Experiments}
\label{sec:exp}

\begin{figure*}[ht]
    \begin{subfigure}{\textwidth}
        \includegraphics[width=\textwidth]{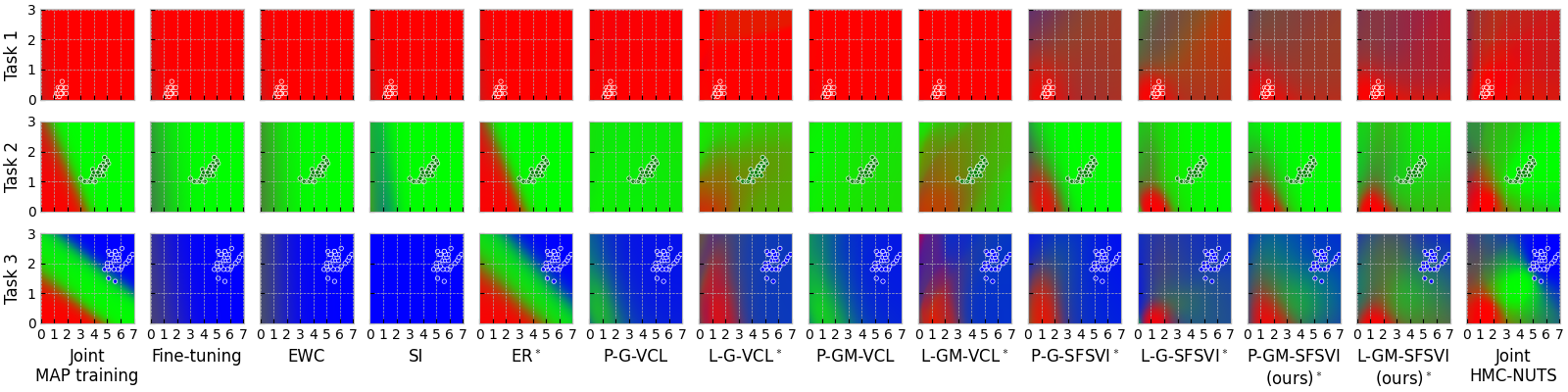}
        \caption{CI Split 2D Iris. The x-axis is the petal length (cm) and the y-axis is the petal width (cm). The pseudocolor plot shows the prediction probabilities mapped to the red, green and blue color channels, and the dots show the observed data. P-GM-SFSVI and L-GM-SFSVI could remember the second class (green) better than their Gaussian counterparts.}
        \centering
    \end{subfigure}
    \begin{subfigure}{\textwidth}
        \includegraphics[width=\textwidth]{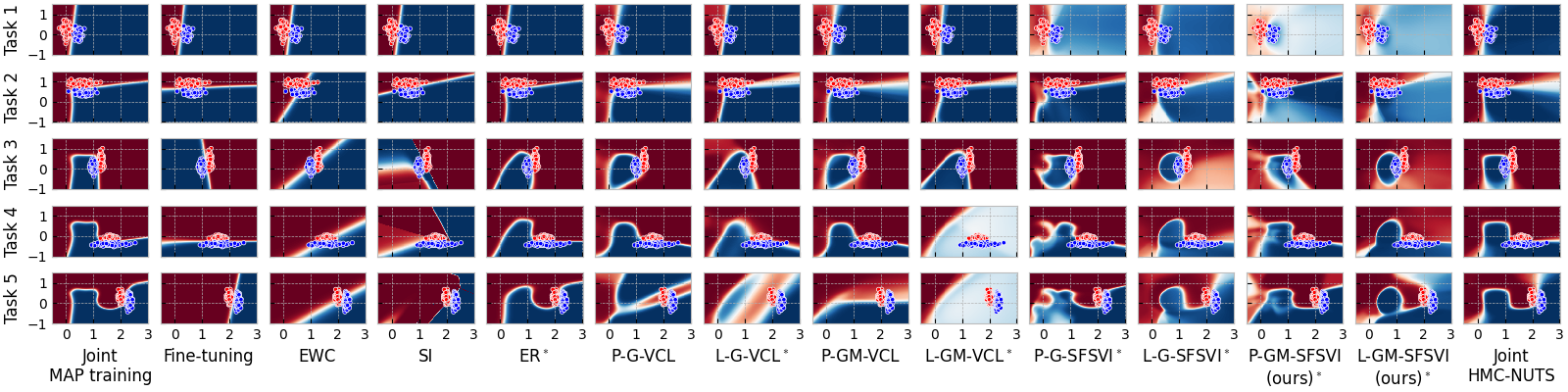}
        \caption{DI Sinusoid. The x-axis and y-axis are the two axes for the synthetic data. The pseudocolor plot shows the prediction probabilities with a red-blue colormap, and the dots show the observed data. The predictions of L-G-SFSVI and L-GM-SFSVI are the most similar to joint HMC-NUTS.}
        \centering
    \end{subfigure}
    \caption{Visualizations of prediction probabilities on DI Sinusoid and CI Split 2D Iris. All Gaussian mixture methods use 3 Gaussian components. \textsuperscript{*}A coreset with 16 data points per task is used.}
    \label{fig:viz}
\end{figure*}

\begin{table*}[ht]
    \caption{Testing final average accuracy for the methods with and without a pre-trained feature extractor. Gaussian mixture methods use 3 Gaussian components. \textsuperscript{*}A coreset with 256 data points per task is used.}
    \label{tab:faa}
    \begin{subtable}{\textwidth}
        \caption{With a fixed pre-trained feature extractor}
        \centering
        \begin{tabular}{lrrrrrrrrrr}
            \toprule
            \header{Method} & \header{CI Split MNIST} & \header{CI Split CIFAR-10} & \header{CI Split HAM-8} & \header{DI Split MNIST} & \header{DI Split CIFAR-8} & \header{DI Split HAM-6}\\
            \midrule
            Joint MAP training & 95.1077 & 76.2500 & 46.9048 & 93.4338 & 96.6250 & 71.9917\\
            Fine-tuning & 19.8382 & 19.0400 & 23.4756 & 64.4789 & 89.6250 & 63.5157\\
            \midrule
            EWC & 73.3744 & 31.9400 & 27.5986 & 82.9540 & 94.6625 & 68.9216\\
            SI & 30.2364 & 27.0500 & 25.7458 & 79.4063 & 95.0750 & 68.1758\\
            ER\textsuperscript{*} & 90.1657 & 62.5900 & \textbf{54.4761} & 92.0569 & 94.5500 & 77.0459\\
            \midrule
            P-G-VCL & 25.2223 & 18.9600 & 22.5610 & 79.8042 & 92.4625 & 64.6759\\
            L-G-VCL\textsuperscript{*} & 90.2739 & 64.3100 & 47.0426 & 91.9639 & 94.6125 & 76.2790\\
            P-GM-VCL & 27.6612 & 19.0000 & 22.8659 & 77.4632 & 92.8125 & 65.4968\\
            L-GM-VCL\textsuperscript{*} & 90.6559 & 64.6800 & 53.3457 & 92.0008 & 94.8125 & 77.5767\\
            P-G-SFSVI\textsuperscript{*} & 88.7514 & 52.6000 & 39.3425 & 91.1877 & \textbf{96.4000} & 72.1336\\
            L-G-SFSVI\textsuperscript{*} & 91.3509 & \textbf{66.1000} & 38.0762 & 91.6020 & 95.5125 & 75.5832\\
            P-GM-SFSVI (ours)\textsuperscript{*} & 89.2844 & 53.8200 & 49.3332 & 91.3151 & 96.3000 & 74.0918\\
            L-GM-SFSVI (ours)\textsuperscript{*} & \textbf{92.0564} & 65.8400 & 50.3536 & \textbf{92.1971} & 95.7375 & \textbf{78.2116}\\
            \bottomrule
        \end{tabular}
    \end{subtable}
    \newline
    \vspace*{0.2cm}
    \newline
    \begin{subtable}{\textwidth}
        \caption{Without a fixed pre-trained feature extractor}
        \centering
        \begin{tabular}{lrrrrrrrrr}
            \toprule
            \header{\multirow{2}{*}{Method}} & \multicolumn{2}{c}{CI Split MNIST} & \multicolumn{2}{c}{DI Split MNIST} & \header{CI Split CIFAR-10} & \header{DI Split CIFAR-8} & \header{CI Split HAM-8} & \header{DI Split HAM-6}\\
            & \header{FCNN} & \header{CNN} & \header{FCNN} & \header{CNN} & \header{CNN} & \header{CNN} & \header{CNN} & \header{CNN}\\
            \midrule
            Joint MAP training & 97.9428 & 99.0519 & 98.6771 & 99.4347 & 76.1400 & 95.9750 & 45.1051 & 75.5069\\
            Fine-tuning & 19.8891 & 19.9496 & 54.9867 & 61.9545 & 19.2500 & 87.1625 & 24.0854 & 65.0410\\
            \midrule
            EWC & 21.5787 & 19.9697 & 68.4412 & 72.9213 & 19.8300 & 89.2625 & 24.0991 & 68.0351\\
            SI & 20.7597 & 20.3879 & 60.2053 & 72.4020 & 19.8900 & 89.0250 & 23.7943 & 65.7724\\
            ER\textsuperscript{*} & 90.5011 & 96.6936 & 94.8631 & 98.3537 & 43.8900 & 90.7375 & \textbf{49.6599} & 74.1218\\
            \midrule
            P-G-VCL & 10.1765 & 10.1664 & 60.9892 & 71.1924 & 10.0000 & 88.6750 & 13.4146 & 63.8091\\
            L-G-VCL\textsuperscript{*} & 43.5818 & 80.8643 & 94.2151 & 97.5226 & 20.3400 & 89.8000 & 13.4146 & 63.8091\\
            P-GM-VCL & 10.1765 & 9.8235 & 59.2359 & 71.7897 & 10.0000 & 88.1125 & 14.0244 & 63.8091\\
            L-GM-VCL\textsuperscript{*} & 44.2003 & 79.8097 & 94.7356 & 97.2071 & 10.0000 & 90.2875 & 21.3085 & 63.8091\\
            P-G-SFSVI\textsuperscript{*} & 41.0995 & 29.8362 & 80.8738 & 94.5080 & 18.7300 & 88.9875 & 17.0328 & 63.8091\\
            L-G-SFSVI\textsuperscript{*} & \textbf{92.3369} & 96.9623 & 95.9703 & 98.5243 & 45.4800 & 91.4250 & 37.9737 & 75.7759\\
            P-GM-SFSVI (ours)\textsuperscript{*} & 40.3262 & 28.8239 & 79.3560 & 83.1516 & 18.5500 & 89.3125 & 16.9440 & 63.8091\\
            L-GM-SFSVI (ours)\textsuperscript{*} & 92.2116 & \textbf{97.2589} & \textbf{96.3143} & \textbf{98.3821} & \textbf{45.4900} & \textbf{91.9125} & 44.0456 & \textbf{78.7513}\\
            \bottomrule
        \end{tabular}
    \end{subtable}
\end{table*}

We first perform experiments on small task sequences to visualize the prediction probabilities. Then, we perform experiments on image task sequences and compare the final average accuracy of different methods. Each task sequence has training, validation and testing dataset sequences. Data used, methods compared and results obtained are described below. More details can be found in Appendix \ref{sec:expdetails}. Documented and reproducible code is available under an MIT Licence at \repourl.

\subsection{Data}

CI indicates class-incremental learning, and DI indicates domain-incremental learning.

The task sequences used for the visualizations are:
\begin{itemize}
    \item \textbf{CI Split 2D Iris}: A task sequence for flower species classification based on Iris and consists of 3 tasks, each with 1 class.
    \item \textbf{DI Sinusoid}: A synthetic task sequence used in \cite{rudner_continual_2022} and consists of 5 tasks, each with 2 classes.
\end{itemize}

The image task sequences used are:
\begin{itemize}
    \item \textbf{CI Split MNIST} and \textbf{DI Split MNIST}: A task sequence for handwritten digit classification based on MNIST and consists of 5 tasks, each with 2 classes. In DI Split MNIST, the classes are set to even and odd. The pre-training task for them is \textbf{EMNIST Letters}, which has no classes in common with them.
    \item \textbf{CI Split CIFAR-10} and \textbf{DI Split CIFAR-8}: A task sequence for natural image classification based on CIFAR-10 and consists of 5 and 4 tasks respectively, each with 2 classes. For DI Split CIFAR-8, 2 animal classes "bird" and "frog" are removed from CIFAR-10, and the classes are set to "vehicle" and "animal". The pre-training task for them is \textbf{CIFAR-100}, which has no classes in common with them.
    \item \textbf{CI Split HAM-8} and \textbf{DI Split HAM-6}: A task sequence for skin condition classification based on HAM10000, which we rename to HAM-8 based on the number of classes, and consists of 4 and 3 tasks respectively, each with 2 classes. For domain-incremental learning, 1 benign class "vascular lesion" and 1 indeterminate class "actinic keratosis" are removed from HAM-8, and the classes are set to "benign" and "malignant". The pre-training task for them is BCN20000, which we rename to \textbf{BCN-12} based on the number of classes, and has classes in common with HAM-8, but they are skin images from different populations.
\end{itemize}

\subsection{Methods}

Joint MAP training and fine-tuning serve as the best and worst reference methods, respectively. SVI method names have 3 parts indicating respectively prior-focused (P) or likelihood-focused (L), Gaussian (G) or Gaussian mixture (GM) and parameter-space (VCL) or function-space (SFSVI). We also compare with SMI methods EWC with Huszár's corrected penalty \parencite{huszar_quadratic_2017,huszar_note_2018}, SI \parencite{zenke_continual_2017} and ER \parencite{chaudhry_continual_2019}. In visualizations, we also use as a reference method joint Hamiltonian Monte Carlo - No U Turn Sampler (HMC-NUTS), which is regarded as the gold standard Bayesian method. However, since it is not scalable, it is not used in image task sequences.

First, experiments are performed with a fixed pre-trained feature extractor. A neural network is pre-trained on a similar task, and all the layers except the final layer are fixed. Then, continual learning is performed on the final layer. This significantly reduces the computational burden, especially for SFSVI methods. Moreover, since the coreset consists of only features, they become privacy-preserving as well. Then, experiments are performed without a fixed pre-trained feature extractor, i.e. continual learning is performed on the whole neural network. Hyperparameter tuning is done for EWC and SI, and the testing final average accuracy is evaluated for all the methods. The testing final average accuracy is the average of the accuracies on all the testing datasets of a task sequence.

\subsection{Results}

We are interested in comparing different sequential variational inference methods in terms of final average accuracy, specifically prior-focused (P) vs likelihood-focused (L) methods, Gaussian (G) vs Gaussian mixture (GM) methods and parameter-space (VCL) vs function-space (SFSVI) methods.

The visualizations of prediction probabilities for all the methods are shown in Figure \ref{fig:viz}. Joint MAP training serves as the reference for all SMI methods, which use MAP prediction, while Joint HMC-NUTS serves as the reference for all SVI methods, which use Bayesian model averaging for prediction. Among SMI methods, ER remains a very competitive method. Among SVI methods, for CI Split 2D Iris, P-GM-SFSVI and L-GM-SFSVI methods give predictions most similar to joint HMC-NUTS, while for DI Sinusoid, L-G-SFSVI and L-GM-SFSVI give predictions most similar to joint HMC-NUTS.

The testing final average accuracy for all the experiments on image task sequences are shown in Table \ref{tab:faa}. As before, we note that ER remains a strong competitor among SMI methods. Although \cite{rudner_continual_2022} reports a high final average accuracy for P-G-SFSVI with an FCNN on CI Split MNIST, we were not able to reproduce similar results. Instead, we find its likelihood-focused counterpart L-G-SFSVI easier to achieve a high final average accuracy. On most of the task sequences, especially without a pre-trained feature extractor, L-GM-SFSVI achieves the highest final average accuracy.

\section{Conclusion}
\label{sec:conc}

We proposed a method that performs sequential function-space variational inference using a Gaussian mixture variational distribution with two variants: prior-focused  and likelihood-focused. We compared them with other types of sequential variational inference methods on class- and domain-incremental task sequences with or without a fixed pre-trained feature extractor and showed empirically that the likelihood-focused Gaussian mixture sequential function-space variational inference achieves the best final average accuracy. Our experiments also provide evidence that in general, likelihood-focused methods perform significantly better than their prior-focused counterparts. We hope that our work revealed more insights into sequential variational inference methods.

\appendices

\section{Experiment Details}
\label{sec:expdetails}

\begin{table*}[ht]
    \caption{Base learning rate, batch size and number of epochs for all tasks. For methods which use a coreset, the coreset batch size is equal to the batch size. The base learning rate for SFSVI methods is 0.1 times the stated rate. The batch size for likelihood-focused methods is half the stated size as the current batch and coreset batch are concatenated in these methods.}
    \label{tab:hp}
        \centering
        \begin{tabular}{lrrr}
            \toprule
            \header{Method} & \header{Base learning rate} & \header{Batch size} & \header{Number of epochs}\\
            \midrule
            \textbf{Visualization experiments}\\
            Methods in CI Split 2D Iris & 0.1 & 16 & 100\\
            Methods in DI Sinusoid & 0.1 & 16 & 100\\
            \midrule
            \textbf{Evaluation experiments with pre-training}\\
            Pre-training on EMNIST Letters & 0.01 & 64 & 20\\
            Pre-training on CIFAR-100 & 0.001 & 64 & 20\\
            Pre-training on BCN-12 & 0.0001 & 64 & 20\\
            Methods in CI Split MNIST & 0.01 & 64 & 20\\
            Methods in CI Split CIFAR-10 & 0.01 & 64 & 20\\
            Methods in CI Split HAM-8 & 0.01 & 64 & 20\\
            Methods in DI Split MNIST & 0.01 & 64 & 20\\
            Methods in DI Split CIFAR-8 & 0.01 & 64 & 20\\
            Methods in DI Split HAM-6 & 0.01 & 64 & 20\\
            \midrule
            \textbf{Evaluation experiments without pre-training}\\
            Methods in CI Split MNIST with FCNN & 0.01 & 64 & 20\\
            Methods in CI Split MNIST with CNN & 0.01 & 32 & 20\\
            Methods in DI Split MNIST with FCNN & 0.01 & 64 & 20\\
            Methods in DI Split MNIST with CNN & 0.01 & 32 & 20\\
            Methods in CI Split CIFAR-10 with CNN & 0.01 & 32 & 20\\
            Methods in DI Split CIFAR-8 with CNN & 0.01 & 32 & 20\\
            Methods in CI Split HAM-8 with CNN & 0.01 & 32 & 20\\
            Methods in DI Split HAM-6 with CNN & 0.01 & 32 & 20\\
            \bottomrule
        \end{tabular}
\end{table*}

\subsection{Data Preparation}

For DI Sinusoid, which has 5 tasks, each with 2 classes, training, validation and testing dataset sequences are prepared by generating 100 points from a bivariate Gaussian distribution with a diagonal covariance matrix for each class and task. For CI Split 2D Iris, Iris with two features "petal length" and "petal width" is split into training and testing datasets with 20\% testing size, and then the training dataset into training and validation datasets with 20\% validation size, so the training, validation and testing proportions are 64\%, 16\% and 20\%, respectively. Finally, each dataset is split by class into a dataset sequence.

For EMNIST Letters, CIFAR-100 and task sequences based on MNIST and CIFAR-10, training and testing datasets are available from PyTorch, so the training dataset is split into training and validation datasets with 20\% validation size. Each dataset is then split by class into a dataset sequence.

For BCN-12 and HAM-8, the \(640\times450\) images are resized to \(32\times32\) with Lanczos interpolation. For all image data, the pixel values are divided by 255 so that they take values between 0 and 1. Data augmentation (e.g. flipping and cropping) is not performed.

\subsection{Neural Network Architectures}

The fully connected neural network used for DI Sinusoid and CI Split 2D Iris has 2 hidden layers each of 16 nodes. All the hidden nodes use swish activation.

The pre-trained neural network for both CI Split MNIST and DI Split MNIST has 2 convolutional layers and 2 dense layers, totaling 4 layers. Each convolutional layer has 32 \(3\times3\) filters and is followed by group normalization with 32 groups, swish activation and average pooling with a size of \(2\times2\). The hidden dense layer has 64 nodes with swish activation. Thus, the feature dimension is 64.

The pre-trained neural network for CI Split CIFAR-10, CI Split HAM-8, DI Split CIFAR-8 and DI Split HAM-6 has 17 convolutional layers and 1 dense layer, totaling 18 layers. Each convolutional layer is followed by group normalization with 32 groups and swish activation. The 2nd to the 17th convolutional layers are arranged into 8 residual blocks, each with 2 convolutional layers, and every 2 residual blocks are followed by average pooling with a size of \(2\times2\). The numbers of filters for the 17 convolutional layers are 32, 64, 64, 64, 64, 128, 128, 128, 128, 256, 256, 256, 256, 512, 512, 512 and 512, respectively, and the filter sizes are all \(3\times3\). Thus, the feature dimension is 512.

For the experiments on CI Split MNIST and DI Split MNIST without a fixed pre-trained feature extractor, the FCNN consists of 2 hidden layers each of 256 nodes, and the CNN is the same as in the pre-trained neural network except that the hidden dense layer has 32 nodes instead of 64 nodes. For the experiments on CI Split CIFAR-10, DI Split CIFAR-8, CI Split HAM-8 and DI Split HAM-6, the CNN has 4 convolutional layers and 2 dense layers, totaling 6 layers. The numbers of filters for the convolutional layers are 32, 32, 64 and 64, and the hidden dense layer has 64 nodes. All the hidden nodes use swish activation.

\subsection{Training}

In all experiments, the prior PDF at time 1 is a standard Gaussian PDF, and an Adam optimizer is used with a one-cycle learning rate schedule. For pre-training tasks and class-incremental task sequences, categorical cross entropy is used, while for domain-incremental task sequences, binary cross entropy is used. For pre-training on BCN-12, which has severe class imbalance, a weighted cross entropy  \(-\sum_{i=1}^k\frac m{n_i}p_i\ln q_i\) is used, where \(p_i\) is the true label indicator and \(q_i\) is the predicted probability, \(n_i\) is the frequency of the \(i\)-th class and \(m=\min\{n_1,n_2,\ldots,n_k\}\).

The learning rate, batch size and number of epochs for all the training tasks are provided in Table \ref{tab:hp}. In EWC and SI, there is a regularization strength \(\lambda\). SI also has a damping factor \(\xi\). They are tuned based on the validation final average accuracy via grid search:
\begin{itemize}
    \item EWC: \(\lambda\in\{1,10,100,1000,10000\}\)
    \item SI: \(\lambda\in\{1,10,100,1000,10000\},\xi\in\{0.1,1.0,10\}\)
\end{itemize}

For SFSVI methods on DI Sinusoid and CI Split 2D Iris, the inducing points are randomly generated from a uniform distribution whose boundaries are determined by the minimum and maximum values of the training input data across all tasks. For image task sequences with pre-training, they are set to -1 and 6. For image task sequences without pre-training, they are set to 0 and 1.

The joint HMC-NUTS used in visualization experiments is run by using all the current and previous data for the log un-normalized posterior PDF and running 64 chains on 64 CPU cores, as opposed to other methods which are run on a GPU. The inverse mass matrix and step size are determined by using a warmup algorithm called window adaptation, which is provided by BlackJAX, and the chains are run for 100 steps after the warmup phase. Only the final sample of size 64 is taken as the posterior sample.

\ack

\printbibliography

\bio

\end{document}